\documentclass[AMA,STIX1COL]{WileyNJD-v2}

\usepackage{graphicx} 

\usepackage{amsmath}
\usepackage{mathtools}
\usepackage{amssymb}
\usepackage{siunitx}
\usepackage{enumitem}
\usepackage[T1]{fontenc}
\usepackage[ansinew]{inputenc}

\articletype{Article Type}%

\received{26 April 2016}
\revised{6 June 2016}
\accepted{6 June 2016}

\begin{document}

\title{Unscented Kalman filter with stable embedding for simple, accurate and computationally efficient state estimation of systems on manifolds in Euclidean space}

\author{Jae-Hyeon Park}

\author{Dong Eui Chang*}

\authormark{Jae-Hyeon Park, and Dong Eui Chang}

\address{\orgdiv{School of Electrical Engineering}, \orgname{Korea Advanced Institute of Science and Technology}, \orgaddress{\state{Daejeon}, \country{Korea}}}

\corres{Dong Eui Chang, School of Electrical Engineering, Korea Advanced Institute of Science and Technology, 291 Daehak-ro, Yuseong-gu, Daejeon 34141, Republic of Korea. \email{dechang@kaist.ac.kr}}

\abstract[Summary]{This paper proposes a simple, accurate and computationally efficient method to apply the ordinary unscented Kalman filter developed in Euclidean space to  systems whose dynamics evolve on manifolds. We use the mathematical theory called stable embedding to make a variant of unscented Kalman filter that keeps state estimates in close proximity to the manifold while exhibiting excellent estimation performance. We confirm the performance of our devised filter by applying it to the satellite system model and comparing the performance with other unscented Kalman filters devised specifically for systems on manifolds. Our devised filter has a low estimation error, keeps the state estimates in close proximity to the manifold as expected, and consumes a minor amount of computation time. Also our devised filter is simple and easy to use because our filter directly employs the off-the-shelf standard unscented Kalman filter devised in Euclidean space without any particular manifold-structure-preserving discretization method or  coordinate transformation.}


\jnlcitation{\cname{%
\author{J. Park},
\author{D. E. Chang}} 
 (\cyear{2021}), 
\ctitle{Unscented Kalman filter with stable embedding for simple, accurate and computationally efficient state estimation of systems on manifolds in Euclidean space}, \cjournal{Int J Robust Nonlinear Control.}, \cvol{2021}.}

\maketitle


\section{Introduction}\label{Introduction}

Filtering and state estimation is an important research field that is essential in various areas of industry such as process control, tracking, navigation, and robot control. 
Recently, researchers are designing filtering and state estimation techniques specifically devised for systems on manifolds. 
This is because many systems are inherently defined on manifolds, and applying filters devised in Euclidean space directly to the systems on manifolds may cause inaccuracy and divergence in state estimation. 
Examples of systems on manifolds are quadrotors,\cite{fresk2013full} winged unmanned aerial vehicles,\cite{guo2019low} satellites,\cite{soken2014robust} robot arms,\cite{jeon2009kinematic} and arrival tracking in antenna array processing.\cite{rentmeesters2010filtering} 
One of the major drawbacks of working with systems on manifolds is that it does not well accommodate many tools available for systems in Euclidean space. 
There have been efforts to generalize the unscented Kalman filter (UKF) to the filtering and state estimation of systems on manifolds because it is one of the popular filtering techniques for nonlinear systems.

The UKF was first devised for accurate filtering and state estimation of nonlinear systems on Euclidean space. 
It became one of the major estimation techniques because it can approximate the statistics of state estimates up to the second order.\cite{julier1997new} 
The detailed explanation of UKF devised in Euclidean space can be found in the paper by Wan and Van der Merwe\cite{wan2000unscented} and we will call it the standard UKF to distinguish it from the variants of UKF later introduced. 
The direct application of the standard UKF to systems on manifolds may not give accurate estimates, because the standard UKF is derived in Euclidean space, so state estimates deviate from the manifolds during UKF process. 
There are variants of UKF that have been developed for filtering and state estimation of systems on manifolds. 
Sipos \cite{sipos2008application} devised an UKF for systems on Lie group, which is a special family of manifolds with group structure. 
This UKF used the property of Lie algebra which is a tangent space at the identity of Lie group and defined sigma points of unscented transform on the Lie algebra. 
Hauberg et al. \cite{hauberg2013unscented} devised an UKF for systems on Riemannian manifold, which is a family of differentiable manifolds with a Riemannian metric defined.
This UKF defined sigma points of unscented transform on the tangent space of the manifold and employed geodesics and parallel transformations. 
These UKFs are practical since many systems of robotics are defined on sets that are both Lie group and Riemannian manifold. i.e. special orthogonal group, special Euclidean group, and n-sphere.
Brossard et al. \cite{brossard2020code} used a similar framework of UKF on manifolds as Hauberg et al. and made an open-source code that can be used for practitioners not being familiar with manifolds and Lie groups. 
Cantelobre et al. \cite{cantelobre2020real} applied the method of Brossard et al. on position estimation of autonomous underwater vehicles.
Sj{\o}berg and Egeland \cite{sjberg2021lie} developed an UKF for matrix Lie groups.
These variants of UKF require special discretized form of systems and manifold-structure-preserving computation steps which can be computationally expensive.

In this article, we propose a method to directly apply the standard UKF devised in Euclidean space to systems on manifolds. 
We use the mathematical theory called {\it stable embedding} that was first founded in the paper by Chang et al.\cite{chang2016feedback} 
Stable embedding was previously applied to extended Kalman filter (EKF) of systems defined on special orthogonal group and the devised filter was computationally efficient. \cite{park2021transversely}
We extend systems on manifolds to an ambient Euclidean space in such a way that the manifold becomes an invariant attractor or an asymptotically stable invariant set for the extended system in the Euclidean space. 
Then we discretize the modified system using the ordinary Euler discretization method and apply the standard UKF derived in Euclidean space to the resultant discrete-time system. 
The reason why we use the Euler method is that it is the simplest one, but one could also use any off-the-shelf discretization method such as the Runge-Kutta 4th-order method which would only increase the accuracy of our proposed UKF. 
By inheriting the stability property of the continuous-time extended system before the discretization, this UKF has high estimation accuracy thanks to state estimates staying in close proximity of the manifold. 
We use the satellite attitude system model \cite{kristiansen2005satellite} to show the excellent performance of our devised UKF. 
The satellite system model has its attitude quaternion state defined on 3-sphere $\operatorname{S}^3=\{\mathbf{q} \in \mathbb{H} \mid \left \| \mathbf{q} \right \|=1\}$, where $\mathbb{H}$ is the set of quaternions which is isomorphic to $\mathbb{R}^{4}$ as a vector space. 
We have chosen the system defined on $\operatorname{S}^3$, because the state space $\operatorname{S}^3$ is both a Lie group and a Riemannian manifold. 
We  compare our UKF with other UKFs designed for systems on manifolds. 
There are also Kalman filters available that apply to  continuous-time systems directly.\cite{frogerais2011various, sarkka2007unscented} They require integration of differential momentum equations which can be computationally complex. 
Hence, applying Kalman filters to  discrete form of continuous-time systems is frequently used in real applications due to fast computation. 
Since, one of the merits of our proposed UKF is fast computation, we choose to discretize continuous-time systems and then apply the UKFs to them.

This paper is organized as follows. Section \ref{sec:UKFs for manifolds} briefly explains the existing methods of UKF for systems on manifolds. 
Section \ref{sec:Stable embedded UKF} describes our method of UKF for systems on manifolds. 
Section \ref{sub:Simulation} shows the performance of our method with a simulation on satellite attitude estimation. 
The conclusion is presented in Section \ref{sec:Conclusion}. 
An appendix is provided at the end of the paper to give a brief review of the standard UKF and a guideline for choosing a value of the stable embedding constant later introduced.


\section{Review of existing methods of  UKF for systems on manifolds}\label{sec:UKFs for manifolds}

Before we explain our method, we briefly introduce existing methods of applying UKF to systems on manifolds. 
Every function and manifold is assumed to be smooth in this paper unless stated otherwise.
Let $M$ be a manifold with dimension $n_M$ in Euclidean space $\mathbb{R}^{n_x}$ where $n_M < n_x$. A general continuous-time system on $M$ is given by 
\begin{subequations} \label{E:continuous system}
\begin{align}
\dot{x} &= f(x, u), \quad x\in M\subseteq \mathbb{R}^{n_x}, u\in \mathbb{R}^{n_u}\\
y &= g(x, u), \quad x \in M\subseteq  \mathbb {R}^{n_x}, u\in \mathbb{R}^{n_u},
\end{align}
\end{subequations}
where $f: M \times \mathbb{R}^{n_u} \rightarrow TM$ is the state differential equation, $x$ is the state, $u$ is the input, $g: M \times \mathbb{R}^{n_u} \rightarrow \mathbb{R}^{n_y}$ is the measurement equation, $y \in \mathbb{R}^{n_y}$ is the measurement. 
Here, $TM$ denotes the tangent bundle  of the manifold $M$. 

Suppose that $f$  can be extended to an open neighborhood of the manifold $M$ in $\mathbb R^{n_x}$. 
If we were to naively apply the standard UKF devised in Euclidean space, we would discretize the system \eqref{E:continuous system} using the ordinary Euler discretization without considering the manifold $M$, and apply the standard UKF. 
In Appendix \ref{Appendix 1}, we give a brief explanation of the standard UKF described in the paper by Wan and Van der Merwe \cite{wan2000unscented} and the typical values of  parameters used in the standard UKF.
Since  the manifold $M$ is not closed on Euclidean addition and subtraction, there arises the problem that $M$ may not be stable for the discretized system, i.e,  discrete-time trajectories of the discretized system may diverge off the manifold $M$. 
Some researchers use the forced projection method to compensate for the state estimates deviating from the manifold. 
This method projects the state estimates $\hat{x}_{k|k}$ at the end of each UKF update to the manifold $M$ by solving the following optimization problem $\hat{x}_{new, k|k} = \arg\max_{x\in M}\left \| x- \hat{x}_{k|k}\right \|^{2}$. 
This least-squares optimization problem can be solved by using an appropriate gradient descent algorithm regarding the manifold, but it can be computationally expensive. 
Detailed explanation of various optimization algorithms can be found in the book by Chong and Zak.\cite{chong2004introduction}

 Some researchers transform the coordinates of state estimate points on manifolds to vector spaces, apply the usual UKF correction steps on the vector spaces and transform the coordinates of points in vector spaces back to the manifolds to get the next state estimates, because many tools of Euclidean space are available in vector spaces. 
 The UKF from the paper by Sipos\cite{sipos2008application} is devised specifically for the case when the manifold $M$ is a Lie group. 
 First, discretize the continuous-time system (\ref{E:continuous system}) defined on the manifold $M$ into a specific form so that the states of discretized system stay on $M$ as follows:
\begin{equation} \label{E:discretized on M}
x_{k+1} = f_D(x_k, u_k),\;  x_k \in M\;  \forall k \in \mathbb N_{\geq 0}. 
\end{equation}
Then apply the variant of UKF that defines sigma points on the Lie algebra of the Lie group, where it is noted that the Lie algebra is a vector space. 
This UKF needs computation of several coordinate transformations between the Lie group and its Lie algebra at each step of estimation, which is computationally intensive and expensive; refer to the book by Hall\cite{hall2015lie} on Lie groups and Lie algebras. 
The UKF from the paper by Hauberg et al.\cite{hauberg2013unscented} is devised for Riemannian manifolds; refer to  the book by Carmo\cite{carmo1992riemannian} on Riemannian manifolds. 
They first discretize the continuous-time system (\ref{E:continuous system}) defined on the manifold $M$ into a specific form so that the states of the discretized system stay on $M$ just like the system (\ref{E:discretized on M}). 
Then apply the variant of UKF that defines sigma points on the tangent spaces at the state estimates on the manifold, where it is noted that the tangent spaces of a manifold are vector spaces. 
This UKF also needs computation of several coordinate transformations between the manifold and its tangent spaces at each step of estimation, which is computationally intensive and expensive. 
In summary, the above two methods proposed in the paper by Sipos\cite{sipos2008application} and the paper by Hauberg et al.\cite{hauberg2013unscented} both need coordinate changes at each step of estimation which can be computationally expensive.


\section{The UKF with stable embedding}\label{sec:Stable embedded UKF}

Here, we devise a variant of UKF that can apply the standard UKF derived in Euclidean space to systems on manifolds without any additional manifold-structure-preserving discretization or coordinate transformations. 
Our method is simple and computationally efficient as good as the standard UKF. 
We use the mathematical theory called stable embedding in the paper by Chang et al.\cite{chang2016feedback}
First, extend the system (\ref{E:continuous system}) from $M$ to $\mathbb{R}^{n_x}$ as follows:
\begin{subequations} \label{E:E system}
\begin{align}
    \dot{x} &= f_{E}(x, u), \quad x\in \mathbb{R}^{n_x}, u\in \mathbb{R}^{n_u}\\
    y &= g_{E}(x, u), \quad x\in \mathbb{R}^{n_x}, u\in \mathbb{R}^{n_u},
\end{align}
\end{subequations}
such that
\begin{subequations} \label{E:E system con}
\begin{align}
    f_{E}(x, u)&= f(x, u), \quad \forall \, x\in M, u\in \mathbb{R}^{n_u}\\
    g_{E}(x, u) &= g(x, u), \quad \forall x \in M, u\in \mathbb{R}^{n_u},
\end{align}
\end{subequations}
where $f_{E}: \mathbb{R}^{n_x} \times \mathbb{R}^{n_u} \rightarrow \mathbb{R}^{n_x}$ and $g_{E}: \mathbb{R}^{n_x} \times \mathbb{R}^{n_u} \rightarrow \mathbb{R}^{n_y}$. 
Condition (\ref{E:E system con}) implies that the extended system (\ref{E:E system}) is identical to the original system (\ref{E:continuous system}) when $x \in M$. 
Thus, $M$ is an invariant set of the extended system (\ref{E:E system}). 
Then, we modify the system (\ref{E:E system}) to make the manifold $M$ asymptotically stable, i.e. all flows of the system starting from a fixed neighborhood of $M$ asymptotically converge to $M$ as time tends to infinity. 
Suppose there is a function $V : \mathbb{R}^{n_x} \rightarrow \mathbb{R}_{\geq0}$ such that 
\begin{equation}\label{E:V condition1}
M = V^{-1}(0),
\end{equation}
and
\begin{equation}\label{E:V condition2}
\langle \nabla V ( x  ),\,  f_{E}(x, u)\rangle=0, \forall \, x\in \mathbb R^{n_x}, u\in \mathbb{R}^{n_u},
\end{equation}
where $V^{-1}(0) = \left \{ x \in \mathbb{R}^{n_x} | V(x) = 0  \right \}$, and $\langle \, ,\rangle$ is the Euclidean inner product. Consider the following modified system
\begin{subequations} \label{E:SE system}
\begin{align}
    \dot{x} &=  f_{E}(x, u) - \alpha\nabla V ( x  ), \quad x\in \mathbb{R}^{n_x}, u\in \mathbb{R}^{n_u}\\
    y &= g_{E}(x, u), \quad x\in \mathbb{R}^{n_x}, u\in \mathbb{R}^{n_u},
\end{align}
\end{subequations}
where $\alpha > 0$ is a constant parameter called stable embedding constant. 
Since the function $V$ attains its minimum value of $0$ on $M$, $ \nabla V ( x  )=0$ for all $x\in M$. 
So the system (\ref{E:SE system}) coincides with the original system (\ref{E:continuous system}) on $M$ and $M$ is an invariant set of (\ref{E:SE system}). 
Since $-\nabla V$ points to the direction of decreasing $V$, it can be shown that $M$ is asymptotically stable for the system (\ref{E:SE system}) on certain conditions. 
More concretely speaking, along each trajectory of the system \eqref{E:SE system}, 
\[
\frac{dV}{dt} = \langle \nabla V, f_E(x,u) - \alpha \nabla V\rangle = -\alpha \|\nabla V\|^2 \leq 0
\]
by \eqref{E:V condition2}. Since $V$ attains its global minimum value $0$, $\nabla V (x) = 0$  in a neighborhood of $M$ if and only if $x\in M$. 
Hence, by Lyapunov's stability theory,  the manifold $M$ is asymptomatically stable for the system \eqref{E:SE system}. 
Also, the above inequality shows that the value of parameter $\alpha$ determines the rate of convergence to $M$; refer to Theorem 1 in the paper by Chang\cite{chang2018controller} for more on the stability conditions for $M$. 
An easy way to construct the function $V$ is to set $V(x) = \| \ell(x) - c\|^2$ when the manifold $M$ is defined by a level set $\ell(x) = c$ of some function $\ell (x)$.

We now add additive noises on the system (\ref{E:SE system}) to produce a stochastic continuous-time system as follows:
\begin{subequations} \label{E:stochastic SE system}
\begin{align}
    \dot{x} &=  f_{E}(x, u) - \alpha\nabla V ( x  ) + w_C, \quad x\in \mathbb{R}^{n_x}, u\in \mathbb{R}^{n_u}, w_C \in \mathbb{R}^{n_x}\\
    y &= g_{E}(x, u) + v_C, \quad x\in \mathbb{R}^{n_x}, u\in \mathbb{R}^{n_u}, v_C \in \mathbb{R}^{n_y},
\end{align}
\end{subequations}
where $x(0) \sim (x_{0}, P_{0})$ is the initial condition of $x$, $w_{C} \sim (0, P_{w,C})$ is the process noise, and  $v_C \sim (0, P_{v, C})$ is the measurement noise. 
We can apply additive noises because the system (\ref{E:SE system}) is globally defined in the Euclidean space. 
Here, $(\mu, P)$ stands for a white noise with mean $\mu$ and covariance $P$. 
It is assumed that $x(0)$, $w_C$, and $v_C$ are mutually uncorrelated white noises.

We discretize the system (\ref{E:stochastic SE system}) using the ordinary Euler discretization with time step $h>0$ to apply the standard UKF derived in Euclidean space as follows:
\begin{subequations} \label{E:discretized system}
\begin{align}
x_{k+1} &=  x_k + h(f_{E}(x_k, u_k) - \alpha\nabla V ( x_k  )) + w_D, \\
y_{k} &= g_{E}(x_{k}, u_{k}) + v_D,
 \end{align}
\end{subequations}
where $x_{k} = x(kh)$, $u_{k} = u(kh)$, and $y_k = y(kh)$ for $k\in \mathbb{N}_{\geq 0}$. 
The stochastic properties are as follows: $x_0 \sim (x_{0}, P_{0})$, $w_{D} \sim (0, P_{w, D})$ where $P_{w, D} \approx  hP_{w, C}$, and $v_{D} \sim (0, P_{v, D})$  where $P_{v, D} \approx P_{v, C}/h$. A detailed explanation of discretization of a stochastic continuous-time system can be found in the book by Lewis et al.\cite[p.84]{lewis2017optimal} 
Then we apply the standard UKF derived in Euclidean space, described in Appendix  \ref{Appendix 1}, to the system \eqref{E:discretized system} directly, which is valid since the system \eqref{E:discretized system} is globally defined in Euclidean space. 

We call the above procedure as the UKF with stable embedding (UKF-SE). 
The big merit of this method is that all trajectories of the discretized form of the continuous-time system \eqref{E:SE system} starting from some neighborhood of $M$ stay in close proximity of $M$.\cite{chang2016feedback} 
In the following section, we apply UKF-SE to the satellite system model.


\section{Simulation}\label{sub:Simulation}

In this section, we conduct numerical simulations with MATLAB to compare performance of UKF-SE with other UKFs. 
The computation platform used for conducting these numerical studies is: CPU - Intel(R) Core(TM) i7-8700 @ 3.20GH, memory - 16GB RAM, and software - MATLAB R2018b. 
We use the satellite system model from the paper by Kristiansen and Nicklasson\cite{kristiansen2005satellite} and the paper by Hajiyev and Soken\cite{hajiyev2014robust} for applying various UKFs. 
Most of the constants and satellite parameters used in the system are from the paper by Hajiyev and Soken.\cite{hajiyev2014robust} 
We organize the parameter values used in our simulation in Table \ref{tab1}.
We denote the Earth-centered inertial reference frame by $I$, the orbit reference frame by $O$, and the body reference frame by $B$ as in the paper by Kristiansen and Nicklasson.\cite{kristiansen2005satellite} 
The stochastic continuous-time satellite system model is given as follows: 
\begin{subequations} \label{E:satellite system}
\begin{align}
\dot{\mathbf{q}}_{BO} &= \frac{1}{2}\Lambda (\mathbf{q}_{BO})[0; \Omega_{BO}], \\ 
\dot{\Omega}_{BI} &= \mathbb{I}^{-1}((\mathbb{I}\Omega_{BI}) \times \Omega_{BI}) + \mathbb{I}^{-1}\tau_{g} + w_{\Omega,C},
\end{align}
\end{subequations}
where $\mathbf{q}_{BO} \in \operatorname{S}^3$ is the satellite attitude quaternion with respect to the orbit frame; $\Omega_{BI} \in \mathbb{R}^3$ is the satellite angular velocity with respect to the inertial frame; 
$\Lambda (\mathbf{q}_{BO})\in \mathbb{R}^{4 \times 4}$ with $\mathbf{q}_{BO} = (q_0, q_1, q_2, q_3)$ is given as 
\begin{align*}
\Lambda (\mathbf{q}_{BO}) = \begin{bmatrix} 
q_0 & -q_1 & -q_2 & -q_3 \\
q_1 & q_0 & -q_3 & q_2 \\
q_2 & q_3 & q_0 & -q_1 \\
q_3 & -q_2 & q_1 & q_0
     \end{bmatrix};
 \end{align*}
$\Omega_{BO} \in \mathbb{R}^3$ is the satellite angular velocity with respect to the orbit frame given as
\begin{equation*}
\Omega_{BO}=\Omega_{BI} + A_{BO} \begin{bmatrix} 0 & -\omega_0 & 0     \end{bmatrix}^T;
\end{equation*}
 $\tau_{g}$ is the gravity gradient torque given as
\begin{equation*}
\tau_{g} = 3 \omega_0^2 A_{BO,3} \times (\mathbb{I} A_{BO,3});
\end{equation*}
 $\mathbb{I}$ is the moment of the inertia matrix of the satellite;  $w_{\Omega,C}\in \mathbb{R}^{3}$ is the process noise which follows the Gaussian distribution $N(0, \sigma^2_{\Omega,C} \cdot I_{3})$ with mean 0 and covariance $\sigma^2_{\Omega,C} \cdot I_{3}$;
$A_{BO}\in \mathbb{R}^{3 \times 3}$ is the rotation matrix of $\mathbf{q}_{BO} = (q_0, q_1, q_2, q_3)$ given as 
\begin{align*}
A_{BO} = \begin{bmatrix} 
(q_0^2 + q_1^2 - q_2^2 - q_3^2) & 2(q_1 q_2 - q_0 q_3) & 2(q_1 q_3 + q_0 q_2)\\
2(q_1 q_2 + q_0 q_3) & (q_0^2 - q_1^2 + q_2^2 - q_3^2) & 2(q_2 q_3 - q_0 q_1)\\
2(q_1 q_3 - q_0 q_2) & 2(q_2 q_3 + q_0 q_1) & (q_0^2 - q_1^2 - q_2^2 + q_3^2) 
     \end{bmatrix};
 \end{align*}
$\omega_0 \approx (\mu / r^3_0)^{1/2}$ is the angular velocity of the orbit with respect to the inertial frame which is a constant;
$\mu$ is the Earth's gravitational coefficient;
$r_0$ is the distance between the satellite and the Earth;
and $A_{BO,3}(t)$ is the 3rd column of $A_{BO}$. 
The initial state for the satellite system is $(\mathbf{q}_{BO}(0), \Omega_{BI}(0)) \sim N((\overline{\mathbf{q}}_{BO,0},\overline{\Omega}_{BI,0}), P_{0})$ where $\overline{\mathbf{q}}_{BO,0} =  (\cos(\pi/4), \sin(\pi/4), 0, 0)$, $\overline{\Omega}_{BI,0} = (-0.01, 0.01, 0.01) [\si{\radian \per \second}]$, and $P_{0} =\operatorname{diagonal}(10^{-8} \cdot  I_{4},\:  0.01^2 \cdot  I_{3})$. 

 The satellite measurement sensors are composed of a magnetometer and a rate gyro. The magnetometer measurement model is given as follows: 
\begin{align}\label{H:def}
\begin{bmatrix}H_x \\ H_y \\ H_z \end{bmatrix} = A_{BO}\begin{bmatrix}H_1 \\ H_2 \\ H_3 \end{bmatrix} + v_{mag,C},
 \end{align}
where $H_x$, $H_y$, and $H_z$ are the measured Earth's magnetic field vector components in the body frame. 
Here, $H_1$, $H_2$, and $H_3$ are the Earth's magnetic field vector components in the orbit frame as a function of time given as 
\begin{align*}
H_1 &= \frac{M_e}{r^3_0}\left \{ \cos(\omega_0t)[\cos(\epsilon) \sin(i) - \sin(\epsilon) \cos(i)\cos(\omega_et)] -  \sin(\omega_0t)\sin(\epsilon)\sin(\omega_et) \right \}, \\
H_2 &= -\frac{M_e}{r^3_0}[\cos(\epsilon) \cos(i) + \sin(\epsilon) \sin(i)\cos(\omega_et)], \\
H_3 &= \frac{2M_e}{r^3_0}\left \{ \sin(\omega_0t)[\cos(\epsilon) \sin(i) - \sin(\epsilon) \cos(i)\cos(\omega_et)] -  2\sin(\omega_0t)\sin(\epsilon)\sin(\omega_et) \right \},
 \end{align*}
where $M_e$ is the magnetic dipole moment of the Earth, $\epsilon$ is the magnetic dipole tilt, $i$ is the orbit inclination, and $\omega_e$ is the spin rate of the Earth. 
In \eqref{H:def}, $v_{mag,C}\in \mathbb{R}^{3}$ is the magnetometer measurement noise which follows $N(0, \sigma^2_{mag,C} \cdot I_{3})$. The rate gyros measurement model is given as follows:
\begin{align}\label{Omega:def}
\Omega_{BI,meas} = \Omega_{BI} + v_{rate,C},
 \end{align}
where $\Omega_{BI,meas}$ is the measured angular rate of the satellite, and $v_{rate,C}\in \mathbb{R}^{3}$ is the rate gyros measurement noise which follows $N(0, \sigma^2_{rate,C} \cdot I_{3})$. All the measurements are measured with a sampling frequency $f$.

\begin{table}[t]%
\centering
\caption{Parameter values used in our simulation \label{tab1}}%
\begin{tabular*}{500pt}{@{\extracolsep\fill}lll@{\extracolsep\fill}}
\toprule
 \textbf{Notation} & \textbf{Parameter} & \textbf{Value}\\
\midrule
$\mathbb{I}$ & the moment of the inertia matrix of the satellite & $\operatorname{diag}(2.1, 2.0, 1.9)\times10^{-3} \, [\si{\kilogram \,  \metre^{2}}]$  \\
$\sigma_{\Omega,C}$ & the standard deviation of the process noise & $0.01 \, [\si{\radian \per \second}]$  \\
$\mu$ & the Earth's gravitational coefficient & $3.98601\times10^{14} \, [\si{\metre^3 \per \second^2}]$ \\
$r_0$ & the distance between the satellite and the Earth & $6928140 \, [\si{\metre}]$ \\
$M_e$ & the magnetic dipole moment of the Earth & $7.943\times 10^{15} \, [\si{\weber \metre}]$\\ 
$\epsilon$ & the magnetic dipole tilt & $11.7\si{\degree}$ \\
$i$ & the orbit inclination & $97\si{\degree}$ \\
$\omega_e$ & the spin rate of the Earth & $7.29\times10^{-5} \, [\si{\radian \per \second}]$ \\
$\sigma_{mag,C}$ & the standard deviation of the magnetometer measurement noise & $200 \, [\si{\nano\tesla}]$ \\
$\sigma_{rate,C}$ & the standard deviation of the rate gyros measurement noise & $1.84 \times 10^{-4} \, [\si{\radian \per \second}]$\\
$f$ & the sampling frequency of the measurement & $10 \, [\si{\hertz}]$ \\
\bottomrule
\end{tabular*}
\end{table}

The satellite system (\ref{E:satellite system}) has a state $(\mathbf{q}_{BO}, \Omega_{BI})$ in $\operatorname{S}^{3} \times \mathbb{R}^3$ which is 6 dimensional manifold in $\mathbb{H} \times \mathbb{R}^3$. We do not consider $\Omega_{BO}$ as a state because it does not have a differential equation. 
We first extend the system (\ref{E:satellite system}) from $\operatorname{S}^3 \times \mathbb{R}^3$ to $\mathbb{H} \times \mathbb{R}^3$. Define a function $V(\mathbf{q}_{BO}, \Omega_{BI}) : \mathbb{H}\times\mathbb{R}^3 \rightarrow \mathbb{R}_{\geq0}$ as
\begin{equation*}
    V(\mathbf{q}_{BO}, \Omega_{BI}) = \frac{1}{4}(\left \| \mathbf{q}_{BO} \right \|^2 - 1)^2.
\end{equation*}
It is easy to see that $\operatorname{S}^3 \times \mathbb{R}^3 = V^{-1}(0)$. The derivative of $V$ is calculated as $\nabla_{\mathbf{q}_{BO}}{V}=(\left \| \mathbf{q}_{BO} \right \|^2 - 1)\mathbf{q}_{BO}$, and $\nabla_{\Omega_{BI}}{V}=0$. This $V$ satistifes the conditions (\ref{E:V condition1}) and (\ref{E:V condition2}) for $V$ in Section \ref{sec:Stable embedded UKF}. We add $-\alpha\nabla{V}$ where $\alpha>0$ to the extended system to get the final modified system as follows:
\begin{subequations}\label{final:sys:H}
\begin{align}
\dot{\mathbf{q}}_{BO} &= \frac{1}{2}\Lambda (\mathbf{q}_{BO})[0; \Omega_{BO}] - \alpha (\left \| \mathbf{q}_{BO} \right \|^2 - 1)\mathbf{q}_{BO}, \\ 
\dot{\Omega}_{BI} &= \mathbb{I}^{-1}((\mathbb{I}\Omega_{BI}) \times \Omega_{BI}) + \mathbb{I}^{-1}\tau_{g} + w_{\Omega,C},
\end{align}
\end{subequations}
where $(\mathbf{q}_{BO}, \Omega_{BI})  \in \mathbb H \times \mathbb R^3$. 
This is the stably embedded version (\ref{E:SE system}) of the satellite system model (\ref{E:satellite system}), i.e. $\operatorname{S}^3 \times \mathbb{R}^3$ is an invariant manifold of the modified system and is exponentially stable for the modified system; refer to Lemma 1 in Ko et al's paper.\cite{ko2021tracking}
Then, we may apply the standard UKF derived in Euclidean space directly to the system model \eqref{final:sys:H} and the measurement models given in \eqref{H:def} and \eqref{Omega:def}.

Since the satellite system in nature is a continuous-time system, we generate a continuous-time trajectory of the satellite using \eqref{E:satellite system}. 
We integrate \eqref{E:satellite system} for $200 [\si{\second}]$ and take measurements at discrete time steps of $t = kh, k\in \mathbb{N}_{\geq 0}$ with a time step of $h = \frac{1}{f}= 0.1 [\si{\second}]$. 
For each type of UKF, we conduct one thousand Monte Carlo runs with each run simulated for 200 seconds. 
While implementing UKF loops, we use the covariance matrix of the initial state $P_{0} = \operatorname{diagonal}(10^{-8} \cdot  I_{4},\:  0.01 \cdot  I_{3})$ and the covariance matrix of the process noise $P_{w, D} = \operatorname{diagonal}(10^{-8} \cdot  I_{4},\:  h\sigma^2_{\Omega,C} \cdot I_{3})$.
We calculate the mean squared error (MSE) of attitude estimates, and mean distance (MD) between attitude estimates and $\operatorname{S}^3$ which are calculated as follows, respectively:
\begin{subequations} \label{E:MSEMD}
\begin{align}
MSE_{\mathbf{q}} &= E[ \| \mathbf{q}(kh) - \hat{\mathbf{q}}_{k} \|^{2} ]\approx \frac{1}{N}\sum_{i=1}^{N}\| \mathbf{q}^{(i)}(kh) - \hat{\mathbf{q}}_{k}^{(i)} \|^{2}, \\
MD_{\mathbf{q}, \operatorname{S}^3} &=  E[ \left |  \| \hat{\mathbf{q}}_{k} \| - 1  \right |]\approx \frac{1}{N}\sum_{i=1}^{N}\left |  \| \hat{\mathbf{q}}^{(i)}_{k} \| - 1  \right |,
 \end{align}
\end{subequations}
where $\mathbf{q}(kh)$ is the real attitude at time $t=kh$ in the continuous-time trajectory, $\hat{\mathbf{q}}_{k}$ is the attitude estimate at time $t=kh$ using each UKF, $E\left [ \cdot  \right ]$ is defined as the expectation with respect to $N=1000$ simulations, and the superscript $(i)$ stands for the $i$-th trial of the $N$ simulations. 
$MSE_{\mathbf{q}}$ represents estimation error of the attitude estimates. 
$MD_{\mathbf{q}, \operatorname{S}^3}$ represents how close the attitude estimates are from $\operatorname{S}^3$ which is the unit quaternion set.
The best $\alpha$ with a minimum estimation error is dependent on the system and the time step, and a general guideline for choosing a value of $\alpha$ is provided in Appendix  \ref{Appendix 2}.

We conduct comparison simulation of our UKF-SE with other existing methods of applying UKFs to systems on manifolds introduced in Section \ref{sec:UKFs for manifolds}. 
We apply the following types of filters on the sampled discrete-time signals of the satellite system model:

\begin{enumerate} [label=(\roman*)]
\item The standard UKF: 
Extend the satellite system model (\ref{E:satellite system}) from $\operatorname{S}^3 \times \mathbb{R}^3$ to $\mathbb{H} \times \mathbb{R}^3$ and apply the standard UKF discussed in Appendix  \ref{Appendix 1}.

\item Projection by optimization: 
This is the forced projection method introduced in Section \ref{sec:UKFs for manifolds}. 
We project the attitude estimates $\hat{\mathbf{q}}_{k|k}$ to the manifold $\operatorname{S}^3$ at the end of each estimation step of (i) by solving $\hat{\mathbf{q}}_{new, k|k} = \arg\max_{\mathbf{q}\in \operatorname{S}^3}\left \| \mathbf{q}- \hat{\mathbf{q}}_{k|k}\right \|^{2}$. 
We use the Lagrangian multiplier method  with the equality constraint $\left \| \mathbf{q} \right \|^{2} - 1 = 0,\;  \mathbf{q} \in \mathbb{H}$ to solve this optimization problem.\cite{chong2004introduction} We use the MATLAB command `fmincon'  to implement the optimization by Lagrangian algorithm.

\item Projection by formula: 
This is for the specific case when the system is defined on $\operatorname{S}^3$. 
After the update is over in (i), we project the state estimate $\hat{\mathbf{q}}_{k|k}$ to the manifold $\operatorname{S}^3$ by using the formula $\hat{\mathbf{q}}_{new, k|k} = \hat{\mathbf{q}}_{k|k} / \left \| \hat{\mathbf{q}}_{k|k} \right \|$.

\item UKF-SE: This is our proposed UKF. From (\ref{Eq: alpha con1}) in Appendix \ref{Appendix 2} and $h=0.1$, $\alpha$ should satisfy $0 < \alpha < \frac{20}{3}$. We choose $\alpha=2$ for our UKF-SE.

\item Lie algebra: 
This UKF is devised specifically for the systems defined on Lie groups \cite{sipos2008application} as introduced in Section \ref{sec:UKFs for manifolds}. 
The specific discretized form (\ref{E:discretized on M}) for the quaternion attitude system is introduced in the paper by Kraft.\cite{kraft2003quaternion} 
This UKF needs several coordinate transformations between $\operatorname{S}^3$ and its Lie algebra at each estimation step. 
The Lie algebra of $\operatorname{S}^3$ is isomorphic to $\mathbb{R}^3$. 
Given a unit quaternion $\mathbf{q} = q_{real} + \mathbf{q}_{imag} \in \operatorname{S}^3$ where $q_{real}$ is the real part and $\mathbf{q}_{imag}$ is the imaginary part of the unit quaternion, the coordinate transformation of $\mathbf{q}$ into the Lie algebra is given as 
\begin{equation} \label{Lie algebra:1}
\mathbf{v} = \arctan \left (\frac{\left \| \mathbf{q}_{imag} \right \|}{q_{real}}\right) \frac{\mathbf{q}_{imag}}{\left \| \mathbf{q}_{imag} \right \| },
\end{equation}
where $\mathbf{v} \in \mathbb{R}^3$. 
Given an element of the Lie algebra $\mathbf{v} \in \mathbb{R}^3$, the coordinate transformation of $\mathbf{v}$ into $\operatorname{S}^3$ is given as 
\begin{equation} \label{Lie algebra:2}
\mathbf{q} = \cos(\left \| \mathbf{v} \right \|) + \sin(\left \| \mathbf{v} \right \|) \frac{\mathbf{v}}{\left \| \mathbf{v} \right \|},
\end{equation}
where $\mathbf{q} \in \operatorname{S}^3$, and $\mathbf{v}$ is considered as a pure imaginary quaternion. 
The map \eqref{Lie algebra:1} is the inverse map of \eqref{Lie algebra:2}; see Sipos\cite{sipos2008application} for more detail.

\item Tangent space: 
This UKF is devised for the systems defined on Riemannian manifolds \cite{hauberg2013unscented} as introduced in Section \ref{sec:UKFs for manifolds}. 
It needs several coordinate transformations between $\operatorname{S}^3$ and its tangent spaces at each estimation step. 
We introduce the coordinate transformation formulas of the Riemannian manifold $\operatorname{S}^3$, logarithm map and exponential map, from the paper by Hauberg.\cite{hauberg2018directional} 
Let us choose a base point $\mathbf{\mu} \in \operatorname{S}^3$ and denote the tangent space of $\operatorname{S}^3$ at $\mathbf{\mu}$ by $T_{\mathbf{\mu}} \operatorname{S}^3=\{\mathbf{u} \in \mathbb{R}^4 \mid \langle \mathbf{u},\,  \mathbf{\mu}\rangle = 0 \}$. A  point $\mathbf{q} \in \operatorname{S}^3$ can be mapped to $T_{\mathbf{\mu}} \operatorname{S}^3$ via the logarithm map,
\begin{align} \label{tangent space 1}
\log_{\mathbf{\mu}}(\mathbf{q}) &= (\mathbf{q} - \mathbf{\mu}(\mathbf{q}^T \mathbf{\mu}))\frac{\theta}{\sin(\theta)}, \nonumber \\
\theta &= \arccos(\mathbf{q}^T \mathbf{\mu}),
\end{align}
with the convention $0 / \sin(0) = 1$. A tangent vector $\mathbf{u} \in T_{\mathbf{\mu}} \operatorname{S}^3$ can be mapped to $\operatorname{S}^3$ via the exponential map,
\begin{equation} \label{tangent space 2}
\exp_{\mathbf{\mu}}(\mathbf{u}) = \mathbf{\mu} \cos(\left \| \mathbf{u} \right \|) + \sin(\left \| \mathbf{u} \right \|) \frac{\mathbf{u}}{\left \| \mathbf{u} \right \|}.
\end{equation}
We use the MVIRT toolbox \cite{bergmann2017mvirt} to compute these coordinate transformations between $\operatorname{S}^3$ and its tangent spaces, and weighted interpolation of points on $\operatorname{S}^3$. 
\end{enumerate}

\begin{figure}
\centerline{\includegraphics{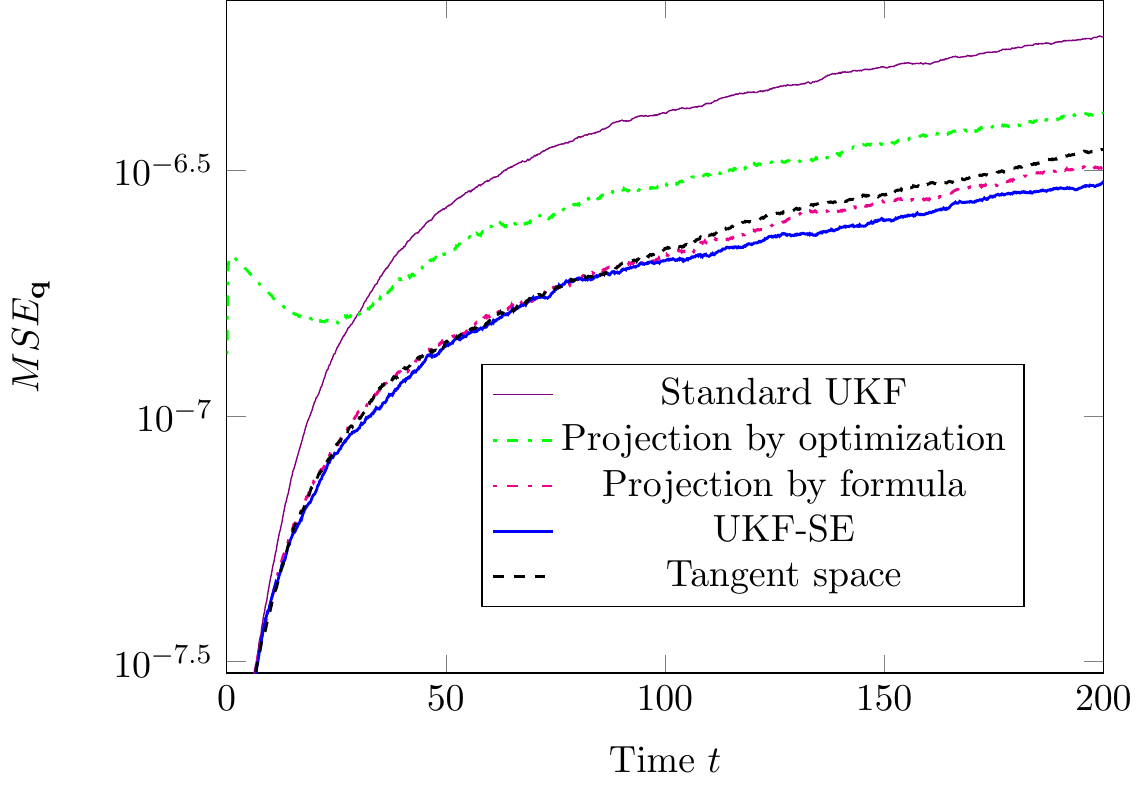}}
\caption{Mean squared error of $\hat{\mathbf{q}}$ versus time $t\, [\si{\second}]$ of each UKFs.}
\label{fig:MSE of q}
\end{figure}

\renewcommand{\arraystretch}{1.3}
\begin{table}[ht]
\centering
\caption{Average $MSE_{\mathbf{q}}$, average $MD_{\mathbf{q}, \operatorname{S}^3}$, and computation time of 1000 simulations.}
\begin{tabular}{| l | c | c | c |}
\hline
Filters & Average $MSE_{\mathbf{q}}$ & Average $MD_{\mathbf{q}, \operatorname{S}^3}$ & Computation time [\si{\second}]\\
\hline
The standard UKF & $5.513\times10^{-7}$ & $4.026\times10^{-4}$ & 798.5 \\
Projection by optimization & $3.873\times10^{-7}$ & $2.310\times10^{-4}$ &  5882.6 \\
Projection by formula & $2.972\times10^{-7}$ & $3.523\times10^{-17}$ & 842.4 \\
UKF-SE (ours) & $2.759\times10^{-7}$ & $1.405\times10^{-6}$ & 869.4 \\
Lie algebra & $1.797\times10^{-4}$ & $5.308\times10^{-15}$  & 3315.1 \\
Tangent space & $3.143\times10^{-7}$ & $3.416\times10^{-17}$ & 5832.2 \\
\hline
\end{tabular}
\label{table:MSE, MD and time}
\end{table}

We plot $MSE_{\mathbf{q}}$ of each UKF as a semi-log plot on Figure \ref{fig:MSE of q}. 
The Lie algebra filter has $MSE_{\mathbf{q}}$ of order $10^{-4}$ which is relatively larger than other filters, because we need to transform the points on $\operatorname{S}^3$ to coordinates of $T_{1} \operatorname{S}^3$ and vice versa at each UKF loop. 
If  points on $\operatorname{S}^3$ are far away from the identity, then they are far away from $T_{1} \operatorname{S}^3$, which naturally  incurs large errors.
So, we do not plot the Lie algebra filter result.
We also do not plot MSE of the angular velocity because there is a negligible difference on MSE of $\Omega_{BI}$ estimates among the UKF's. 
We calculate the average $MSE_{\mathbf{q}}$ along time from $t=150[\si{\second}]$ to $t=200[\si{\second}]$ and record them on the 2nd column of Table \ref{table:MSE, MD and time}. 
The UKF-SE has the lowest estimation error for estimating attitude among the compared UKF's as shown on Figure \ref{fig:MSE of q} and Table \ref{table:MSE, MD and time}, which shows numerically that UKF-SE is the most accurate UKF for systems on manifolds. 
We calculate the average $MD_{\mathbf{q}, \operatorname{S}^3}$ along time from $t=150[\si{\second}]$ to $t=200[\si{\second}]$ and write on the 3rd column of Table \ref{table:MSE, MD and time}.  
The average $MD_{\mathbf{q}, \operatorname{S}^3}$ of UKF-SE is 0.0035 times the average $MD_{\mathbf{q}, \operatorname{S}^3}$ of the standard UKF, which shows that UKF-SE keeps the attitude estimates in close proximity to the manifold due to the stable embedding.
The average $MD_{\mathbf{q}, \operatorname{S}^3}$'s of the projection by formula filter, the Lie algebra filter, and the Tangent space filter are the lowest among the UKF's because they are composed in a specific way that constrains the attitude estimates on the manifold. 
However the average $MD_{\mathbf{q}, \operatorname{S}^3}$'s of them are not exactly 0 because of numerical rounding errors.

We also calculate the computation time of each UKF for 1000 Monte Carlo runs and record them on the 4th column of Table \ref{table:MSE, MD and time}. 
The number of sigma points is 20 for all the simulated UKFs in our paper. We use the concept of precision operations written in the paper by Brent\cite{brent1976fast} to perform a thorough computing complexity analysis of the UKFs. 
Let $M(n)$ be the number of single-precision operations required to multiply $n$-bit integers. 
Elementary operations, such as addition, multiplication, and reciprocals can be evaluated in $O(M(n))$ operations, with a relative error $O(2^{-n})$.\cite{brent1976fast} 
Trigonometric functions, such as sin, cos, arctan can be evaluated in $O(M(n)\log(n))$ operations, with a relative error $O(2^{-n})$.\cite{brent1976fast} 
The UKF-SE shows 1.004 times slower performance than the standard UKF. 
The reason for UKF-SE taking more time than the standard UKF is that UKF-SE  carries out  the additional subtraction of the stable embedding term $\alpha(\left \| \mathbf{q}_k \right \|^2 - 1)\mathbf{q}_k$ for each sigma point in each UKF loop in comparison with the standard UKF. 
The subtraction of the stable embedding term has 17 elementary operations, and there are 20 sigma points. 
Hence, the UKF-SE computes $O(340M(n))$ more operations than the standard UKF in each UKF loop. 
The projection by optimization filter shows the slowest computation time among the simulated filters. 
The reason for this is that the projection by optimization filter solves an optimization problem in each UKF loop. 
The UKF-SE shows a minor difference in computation time from the projection by formula filter. 
The projection by formula filter calculates an additional normalization formula at the end of each UKF loop in comparison with the standard UKF, which is composed of 11 elementary operations. 
Therefore, the projection by formula filter computes $O(11M(n))$ more operations than the standard UKF in each UKF loop. 
It is however noted that a general nonlinear manifold, unlike the 3-sphere $\operatorname{S}^3$, may not even have a closed-form of normalization formula or, may only have computationally expensive normalization algorithms such as polar decomposition for special orthogonal group. 
The Lie algebra filter has 3.81 times slower performance than UKF-SE because the Lie algebra filter computes additional coordinate transformations between the Lie group and its Lie algebra, (\ref{Lie algebra:1}) and (\ref{Lie algebra:2}), for each sigma point at each UKF loop in comparison with the standard UKF. 
Computation of (\ref{Lie algebra:1}) and (\ref{Lie algebra:2}) is composed of 44 elementary operations and 3 trigonometric operations. 
Hence, the Lie algebra filter computes $O((880 + 60\log(n))M(n))$ more operations than the standard UKF at each UKF loop. 
The tangent space filter has 6.71 times slower performance than UKF-SE because the tangent space filter computes additional coordinate transformations between the Riemannian manifold and its tangent spaces, (\ref{tangent space 1}) and (\ref{tangent space 2}), for each sigma point at each UKF loop in comparison with the standard UKF. 
Computation of (\ref{tangent space 1}) and (\ref{tangent space 2}) is composed of 65 elementary operations and 4 trigonometric operations. 
Therefore, the tangent space filter computes $O((1300 + 80\log(n))M(n))$ more operations than the standard UKF in each UKF loop.


\section{Conclusions}\label{sec:Conclusion}

We have devised a simple, accurate, and computationally efficient unscented Kalman filter (UKF) for the state estimation of systems on manifolds. 
The idea is to extend the given system to an ambient Euclidean space, apply the special term which makes the modified system asymptotically stable on the manifold, and apply the standard UKF derived in Euclidean space. 
Our proposed UKF is simple because it uses the off-the-shelf standard UKF devised in Euclidean space directly while keeping the state estimates in close proximity to the manifold. 
We apply this filter to the satellite system model which has attitude state defined on the unit quaternion space. 
We investigate the performance of our devised UKF by comparing with existing methods of applying UKF to systems on manifolds. 
Simulation on the satellite system model shows that our UKF is the most accurate compared to the other UKFs. 
Our UKF also has the minor computation time comparable to the standard UKF devised in Euclidean space, because our filter does not need either any manifold-structure-preserving discretization method or coordinate transformations of other UKFs. 
We plan to expand our stable embedding technique to invariant Kalman filters.\cite{trawny2005indirect}


\section*{Acknowledgments}
This work has been supported by the NRF grant funded by the Korea government (MSIT) (2021R1A2C2010585).

\appendix

\section{The standard UKF} \label{Appendix 1}

The UKF uses the concept of unscented transform to find stochastic properties of state estimates. 
The unscented transform uses sigma points which capture the mean and the covariance of states accurate to the second order.\cite{julier1997new} 
We briefly summarize the standard UKF algorithm devised in Euclidean space described in the paper by Wan and Van der Merwe.\cite{wan2000unscented} 
A discrete-time stochastic system is given as follows:
\begin{align*}
x_{k+1} &= f_D(x_k, u_k, w_k), \\
y_{k} &=  g_D(x_{k}, u_{k}, v_{k}),
 \end{align*}
where $k\in \mathbb{N}_{\geq 0}$, the state is $x_{k} \in \mathbb{R}^{n_x}$, the input is $u_{k} \in \mathbb{R}^{n_u}$, the output is $y_{k} \in \mathbb{R}^{n_y}$, the process noise is $w_{k} \in \mathbb{R}^{n_w}$, the measurement noise is $v_{k} \in \mathbb{R}^{n_v}$, $x_{0} \sim (\overline{x}_{0}, P_{0})$, $w_{k} \sim (0, P_{w, D})$, and $v_{k} \sim (0, P_{v, D})$. 
It is assumed that $x_{0}$, $w_{k}$, and $v_{k}$ are mutually uncorrelated white noises. 
We want state estimates $\hat{x}_{k|k}$ for the system above. 
First, we initialize the state estimate as $\hat{x}_{0|0} = \overline{x}_{0}$ and the covariance estimate as $P_{0|0}=P_{0}$.
Then, we define the augmented state $x_k^a = [x_k^T, w_k^T, v_k^T]^T \in \mathbb{R}^{L}$ for $k\in \mathbb{N}_{\geq 0}$ where $L = n_x + n_w + n_v$. 
The augmented state estimate is given as $\hat{x}_{k|k}^a = [\hat{x}_{k|k}^T, 0 , 0]^T \in \mathbb{R}^{L}$. 
The augmented covariance estimate is given as $P_{k|k}^a=\operatorname{diagonal}(P_{k|k}, P_{w, D}, P_{v, D} )$. 
We iteratively generate sigma points, do time update and do measurement update at each $k\in \mathbb{N}_{\geq 1}$. 
The process for generating sigma points is as follows:
\begin{align*}
&\sigma_{i}^a = [\sigma_i^x, \sigma_i^w, \sigma_i^v] \text{ where } \sigma_i^x \in \mathbb{R}^{n_x}, \sigma_i^w \in \mathbb{R}^{n_w}, \sigma_i^v \in \mathbb{R}^{n_v} \text{ and } i=0,1,2,\cdots ,2L,\\
&\sigma_0^a = \hat{x}_{k-1|k-1}^a,\\
&\sigma_i^a = \hat{x}_{k-1|k-1}^a + \sqrt{L+\lambda}\left [ \sqrt{P_{k-1|k-1}^a} \right ]_i \quad  i=1,\cdots ,L,\\
&\sigma_i^a = \hat{x}_{k-1|k-1}^a - \sqrt{L+\lambda}\left [ \sqrt{P_{k-1|k-1}^a} \right ]_{i-L} \quad  i=L+1,\cdots ,2L,\\
& W_0^{(m)} = \lambda / (L + \lambda), \\
& W_0^{(c)} = \lambda / (L + \lambda) + (1-\gamma^2+\beta), \\
& W_i^{(m)} = W_i^{(c)} = \lambda / {2(L + \lambda)} \quad  i=1,\cdots ,2L,
\end{align*}
where $\left [ \sqrt{P_{k-1|k-1}^a} \right ]_i$  is the $i^{th}$ column of the matrix, and $\lambda=\gamma^2(L+\kappa )-L$. 
The parameters are adopted from the paper by Wan and Van der Merwe\cite{wan2000unscented} which are $\gamma  = 0.01$, $\kappa = 0$, and $\beta=2$.
The process for doing time update is as follows:
\begin{align*}
& \sigma_{i, k|k-1}^x = f_D(\sigma_{i}^x, u_{k-1}, \sigma_{i}^w) \quad  i=0,\cdots ,2L, \\
& \hat{x}_{k|k-1} = \sum_{i=0}^{2L}W_i^{(m)}\sigma_{i, k|k-1}^x, \\
& P_{k|k-1} = \sum_{i=0}^{2L}W_i^{(c)}[\sigma_{i, k|k-1}^x - \hat{x}_{k|k-1}][\sigma_{i, k|k-1}^x - \hat{x}_{k|k-1}]^T, \\
& y_{i, k|k-1} = g_{D}(\sigma_{i}^x, u_{k-1}, \sigma_{i}^v) \quad  i=0,\cdots ,2L, \\
& \hat{y}_{k|k-1} = \sum_{i=0}^{2L}W_i^{(m)} y_{i, k|k-1}.
\end{align*}
The process for doing measurement update is as follows:
\begin{align*}
& P_{yy}=\sum_{i=0}^{2L}W_i^{(c)}[y_{i, k|k-1} - \hat{y}_{k|k-1}][y_{i, k|k-1} - \hat{y}_{k|k-1}]^T, \\
& P_{xy}=\sum_{i=0}^{2L}W_i^{(c)}[\sigma_{i, k|k-1}^x - \hat{x}_{k|k-1}][y_{i, k|k-1} - \hat{y}_{k|k-1}]^T, \\
& K = P_{xy}P_{yy}^{-1}, \\
& \hat{x}_{k|k} = \hat{x}_{k|k-1} + K(y_k - \hat{y}_{k|k-1}), \\
& P_{k|k} = P_{k|k-1} - KP_{yy}K^T,
\end{align*}
where $y_k$ is the measurement and $\hat{x}_{k|k}$ is the state estimate at time step $k$.

\section{A guideline for choosing a value of $\alpha$} \label{Appendix 2}

We are given the following equation with stable embedding (\ref{final:sys:H}) from Section \ref{sub:Simulation}:
\begin{equation*}
\dot{\mathbf{q}}_{BO} = \frac{1}{2}\Lambda (\mathbf{q}_{BO})[0; \Omega_{BO}] - \alpha (\left \| \mathbf{q}_{BO} \right \|^2 - 1)\mathbf{q}_{BO}.
\end{equation*}
We discretize the above equation using the ordinary Euler discretization with time step $h>0$ as below
\begin{align} \label{Eq: discrete dynamics appendix}
\mathbf{q}_{BO, k+1} &=  \mathbf{q}_{BO, k} + h \left ( \frac{1}{2}\Lambda (\mathbf{q}_{BO, k})[0; \Omega_{BO, k}] - \alpha (\left \| \mathbf{q}_{BO, k} \right \|^2 - 1)\mathbf{q}_{BO, k} \right ), \nonumber \\
&=(1 - \alpha h (\left \| \mathbf{q}_{BO, k} \right \|^2 - 1)) \mathbf{q}_{BO, k} + \frac{h}{2}\Lambda (\mathbf{q}_{BO, k})[0; \Omega_{BO, k}] .
\end{align}
\begin{theorem}
If we define the error state as below
\begin{equation*} 
e_k = \left \| \mathbf{q}_{BO, k}\right \|^2 - 1,
\end{equation*}
then the discrete error dynamics of (\ref{Eq: discrete dynamics appendix}) is given as
\begin{equation} \label{Eq: discrete error dynamics}
e_{k+1} = -1 + (1 + e_k) \left ( (1-\alpha h e_k)^2 + \frac{h^2 \left \| \Omega_{BO, k} \right \|^2}{4} \right ).
\end{equation}
\end{theorem}

\begin{proof}
It is easy to show the following identities:
\begin{align*}
&\langle \mathbf{q}_{BO, k} ,\,  \Lambda (\mathbf{q}_{BO, k})[0; \Omega_{BO, k}] \rangle=0, \nonumber \\
&\left \| \Lambda (\mathbf{q}_{BO, k})[0; \Omega_{BO, k}]\right \|^2 = \left \| \mathbf{q}_{BO, k}\right \|^2 \left \| \Omega_{BO, k} \right \|^2.
\end{align*}
With the above identities and the definition of the error $e_k$, 
\begin{align*}
e_{k+1} &= \left \| \mathbf{q}_{BO, k+1}\right \|^2 - 1, \\
&= (1 - \alpha h (\left \| \mathbf{q}_{BO, k} \right \|^2 - 1))^2 \left \| \mathbf{q}_{BO, k}\right \|^2 + \frac{h^2}{4}\left \| \mathbf{q}_{BO, k}\right \|^2 \left \| \Omega_{BO, k} \right \|^2 - 1, 
\end{align*}
which simplifies to \eqref{Eq: discrete error dynamics}.

\end{proof}

The next theorem gives a general guideline for choosing a value of $\alpha$.

\begin{theorem}
Given $\left | e_k \right | \leq \epsilon < 1$, and $\left \| \Omega_{BO, k} \right \| \leq U \;  \forall k$, if we choose $\alpha$, $h$, and $\beta$ that satisfy
\begin{equation} \label{Eq: alpha con1}
0 < \alpha h <\frac{2}{3},
\end{equation}
\begin{equation} \label{Eq: alpha con2}
0 < \beta < \alpha h (1 - \epsilon) (2 - \alpha h \epsilon) \left ( 2 - 2\alpha h - \alpha h (2 - \alpha h) \epsilon + (\alpha h)^2 \epsilon^2 \right ) < 1 ,
\end{equation}
and
\begin{equation} \label{Eq: alpha con3}
\delta(h, \alpha h, \beta)  < \epsilon,
\end{equation}
where
\begin{equation*}
\delta(h, \alpha h, \beta) = h\sqrt{\frac{(1 -2\alpha h (1 - \epsilon) + (\alpha h)^2 \epsilon (1 + \epsilon) ) \epsilon (1 + \epsilon) \frac{U^2}{2} + (1 + \epsilon)^2 \frac{h^2 U^4}{16}}{\alpha h (1 - \epsilon) (2 - \alpha h \epsilon) \left ( 2 - 2\alpha h - \alpha h (2 - \alpha h) \epsilon + (\alpha h)^2 \epsilon^2 \right ) - \beta}},
\end{equation*}
then the set
\begin{equation} \label{Eq: attractor set}
S_\delta = \left \{ e_k \in \mathbb R \mid \left | e_k \right | < \delta(h, \alpha h, \beta) \right \}
\end{equation}
becomes a positively invariant attractor of the discrete error dynamics (\ref{Eq: discrete error dynamics}) with a region of attraction 
\[
S_\epsilon = \left \{ e_k\in \mathbb R \mid \left | e_k \right | \leq \epsilon \right \}.
\]
\end{theorem}

\begin{proof}
From the condition (\ref{Eq: alpha con1}), $\alpha h$ always satisfies the below inequality
\begin{equation*}
0 < \alpha h < \operatorname{min}\left ( \frac{2}{2 \epsilon + 1}, \; \frac{1 + \epsilon - \sqrt{1 - \epsilon^2}}{\epsilon (1 + \epsilon)} \right ),
\end{equation*}
because the infimum of the right side is $\frac{2}{3}$ for $0 < \epsilon < 1$. With $\left | e_k \right | \leq \epsilon < 1$, the following inequalities hold
\begin{align} \label{Eq: cons}
& 0 < 1 - \epsilon \leq 1 + e_k < 1 + \epsilon, \nonumber \\
& 0 < 2 - \alpha h \epsilon \leq 2 - \alpha h e_k , \nonumber \\
& 0 < 2 - 2\alpha h - \alpha h (2 - \alpha h) \epsilon + (\alpha h)^2 \epsilon^2 \leq 2 - 2\alpha h - \alpha h (2 - \alpha h)e_k + (\alpha h)^2 e_k^2. 
\end{align}

Choose a Lyapunov function candidate $V(e_k) = e_k^2$. From (\ref{Eq: discrete error dynamics}) and (\ref{Eq: cons}),
\begin{align}\label{Eq: Lyapunov cal}
V(e_{k+1}) - V(e_k) & = e_{k+1}^2 - e_k^2 \nonumber \\
& = -\alpha h (1 + e_k) (2 - \alpha h e_k) \left ( 2 - 2\alpha h - \alpha h (2 - \alpha h)e_k + (\alpha h)^2 e_k^2 \right ) e_k^2 \nonumber \\
 &\quad  + (1 -2\alpha h (1 + e_k) + (\alpha h)^2 e_k (1 + e_k) ) e_k (1 + e_k) \frac{h^2 \left \| \Omega_{BO, k} \right \|^2}{2} \nonumber \\
 & \quad  + (1 + e_k)^2 \frac{h^4 \left \| \Omega_{BO, k} \right \|^4}{16}  \nonumber \\
& \leq -\alpha h (1 - \epsilon) (2 - \alpha h \epsilon) \left ( 2 - 2\alpha h - \alpha h (2 - \alpha h) \epsilon + (\alpha h)^2 \epsilon^2 \right ) e_k^2 \nonumber \\
&\quad  + (1 -2\alpha h (1 - \epsilon) + (\alpha h)^2 \epsilon (1 + \epsilon) ) \epsilon (1 + \epsilon) \frac{h^2 U^2}{2} + (1 + \epsilon)^2 \frac{h^4 U^4}{16}
\end{align}
since $\left \| \Omega_{BO, k} \right \| \leq U$.

If $\left | e_k \right |$ satisfies the  inequality 
\begin{equation} \label{Eq: ek range 1}
\delta(h, \alpha h, \beta)  < \left | e_k \right | \leq \epsilon,\\
\end{equation}
which is feasible by the condition (\ref{Eq: alpha con3}), then from (\ref{Eq: Lyapunov cal}), $V(e_{k+1}) - V(e_{k}) \leq -\beta e_k^2$. 
Consequently, $V(e_{k+1}) \leq (1-\beta) V(e_k)$ when $e_k$ satisfies (\ref{Eq: ek range 1}). 
Therefore, $e_k$ decreases exponentially in magnitude so as to get into the set $S_\delta$ in  finite time, since $0 < \beta < 1$ from the condition (\ref{Eq: alpha con2})

If $\left | e_k \right |$ satisfies the  inequality 
\begin{equation} \label{Eq: ek range 2}
\delta(h, \alpha h, 0) < \left | e_k \right | \leq  \delta(h, \alpha h, \beta),
\end{equation}
where
\begin{equation*} 
\delta(h, \alpha h, 0) = \left.\begin{matrix}
\delta(h, \alpha h, \beta)
\end{matrix}\right|_{\beta=0} = h\sqrt{\frac{(1 -2\alpha h (1 - \epsilon) + (\alpha h)^2 \epsilon (1 + \epsilon) ) \epsilon (1 + \epsilon) \frac{U^2}{2} + (1 + \epsilon)^2 \frac{h^2 U^4}{16}}{\alpha h (1 - \epsilon) (2 - \alpha h \epsilon) \left ( 2 - 2\alpha h - \alpha h (2 - \alpha h) \epsilon + (\alpha h)^2 \epsilon^2 \right )}},
\end{equation*}
then $V(e_{k+1}) - V(e_{k}) \leq 0$ from (\ref{Eq: Lyapunov cal}), which means $\left | e_{k+1} \right | \leq \left | e_k \right |$. 
Hence, $e_{k+1}\in S_\delta$ if $e_k$ satisfies (\ref{Eq: ek range 2}).

If $\left | e_k \right |$ satisfies the  inequality 
\begin{equation} \label{Eq: ek range 3}
\left | e_k \right | \leq \delta(h, \alpha h, 0),
\end{equation}
then by moving $V(e_k)$ to the right side in (\ref{Eq: Lyapunov cal}) and using (\ref{Eq: ek range 3}),
\begin{align}
V(e_{k+1}) & \leq \left (  1-\alpha h (1 - \epsilon) (2 - \alpha h \epsilon) \left ( 2 - 2\alpha h - \alpha h (2 - \alpha h) \epsilon + (\alpha h)^2 \epsilon^2 \right )\right ) e_k^2 \nonumber \\
&\quad + (1 -2\alpha h (1 - \epsilon) + (\alpha h)^2 \epsilon (1 + \epsilon) ) \epsilon (1 + \epsilon) \frac{h^2 U^2}{2} + (1 + \epsilon)^2 \frac{h^4 U^4}{16} \nonumber \\
&\leq \left (  1-\alpha h (1 - \epsilon) (2 - \alpha h \epsilon) \left ( 2 - 2\alpha h - \alpha h (2 - \alpha h) \epsilon + (\alpha h)^2 \epsilon^2 \right ) \right ) \times \delta(h, \alpha h, 0)^2 \nonumber \\
&\quad + (1 -2\alpha h (1 - \epsilon) + (\alpha h)^2 \epsilon (1 + \epsilon) ) \epsilon (1 + \epsilon) \frac{h^2 U^2}{2} + (1 + \epsilon)^2 \frac{h^4 U^4}{16} \nonumber \\
& = \delta(h, \alpha h, 0)^2,
\end{align}
which implies 
\begin{equation*}
\left | e_{k+1} \right | \leq \delta(h, \alpha h, 0),
\end{equation*}
if $e_k$ satisfies (\ref{Eq: ek range 3}). 
Therefore, $e_{k+1}\in S_\delta$ if $e_k$ satisfies (\ref{Eq: ek range 3}). 
To sum up, every trajectory of the discrete error dynamics (\ref{Eq: discrete error dynamics}) starting in the set $S_\delta$ remains in $S_\delta$ for all forward time, and every trajectory starting in $S_\epsilon$ enters the set $S_\delta$ in finite time and remains in $S_\delta$ from then on. 
In other words, $S_\delta$ is a positively invariant attractor of the error dynamics with $S_\epsilon$ a region of attraction.
\end{proof}

\begin{remark}
For any $0 < \epsilon<1$ and $U>0$, we can find $\alpha$, $h$, and $\beta$ that satisfy the conditions (\ref{Eq: alpha con1}), (\ref{Eq: alpha con2}), and (\ref{Eq: alpha con3}). 
First, we can always find a value for $\alpha h$ that satisfies the third inequality in (\ref{Eq: alpha con2}) by choosing a sufficiently small $\alpha h$, for example. 
Choose a value of $\beta$ that satisfies the first and second inequalities of (\ref{Eq: alpha con2}).
Once $\alpha h$ and $\beta$ are fixed, we can always choose a value of $h$ that satisfies (\ref{Eq: alpha con3}). 
Then, find a value of $\alpha$ that satisfies (\ref{Eq: alpha con1}).
\end{remark}


\nocite{*}
\bibliography{quaternion_UKF_ref}%

\clearpage

\end{document}